\documentclass{article}

\usepackage[english]{babel}

\usepackage[letterpaper,top=2cm,bottom=2cm,left=3cm,right=3cm,marginparwidth=1.75cm]{geometry}

\usepackage{authblk}

\usepackage{graphicx}
\usepackage{cite}
\usepackage{amsmath,amssymb,amsfonts}
\usepackage{algorithmic}
\usepackage{graphicx}
\usepackage{textcomp}
\usepackage{xcolor}
\usepackage{amsthm}
\usepackage[colorlinks=true, allcolors=blue]{hyperref}

\usepackage{natbib}
\usepackage{booktabs}

\newtheorem{theorem}{\bf Theorem}

\newtheorem{lemma}{\bf Lemma}

\newtheorem{corollary}{\bf Corollary}

\title{In-Context Reward Adaptation for Robust Preference Modeling}
\author[1]{Zhenyu Sun}
\author[2]{Zheng Xu}
\author[1]{Ermin Wei}
\affil[1]{Northwestern University}
\affil[2]{Meta Superintelligence Labs}

\date{}

\begin{document}
\maketitle

\begin{abstract}
  Reinforcement Learning from Human Feedback (RLHF) typically relies on static reward models to align Large Language Models with human preferences. However, human values are inherently diverse and heterogeneous, and a single reward model often lacks the robustness required to generalize to unseen preference domains. While existing multi-reward frameworks attempt to address this, they are often restricted to a fixed set of known domains and fail to adapt to unseen human distributions without costly retraining. In this work, we propose In-Context Reward Adaptation, a transformer-based framework designed to model diverse and unseen human preferences on the fly. By leveraging the in-context learning capabilities of transformers, our approach adaptively infers the underlying reward structure from a small set of preference demonstrations. We demonstrate that while a standard transformer architecture is insufficient for this task by characterizing an asymptotic bias to the ground-truth, incorporating human response time as an auxiliary input signal enables the model to successfully adapt to preferences from previously unseen domains. Our findings show that this approach provides a more robust foundation for preference modeling, allowing for the representation of heterogeneous rewards and preference distribution shift, and offering a scalable path toward more flexible human-AI alignment.
\end{abstract}

\section{Introduction}

Reinforcement Learning from Human Feedback (RLHF) has become a central paradigm for aligning large language models (LLMs) with human intent, enabling models to be fine-tuned using reward functions learned from human preference data \cite{ouyang2022training, korbak2023pretraining,shaikh2024aligning}. In most existing RLHF pipelines, a \emph{single} reward model is trained on aggregated preference data and then used to guide policy optimization \cite{ouyang2022training,wang2023aligning}. While this approach has proven effective at scale, it implicitly assumes that human preferences are sufficiently stable and homogeneous to be represented by a single global objective.

In practice, however, human preferences are neither homogeneous nor static. On the one hand, preferences can be dynamic even for the same individual: as users gain experience, encounter new contexts, or adapt their goals over time, their judgments and decision criteria may shift. On the other hand, preferences are highly \emph{heterogeneous} across individuals \cite{lambert2023alignment,li2024enhancing,casper2023open}. Differences in cultural background, personal values, domain expertise, and situational context can lead different humans to evaluate the same model output in drastically different ways. As a result, preference data collected from a population reflects a mixture of diverse and evolving reward structures rather than a single coherent objective.

% This mismatch exposes a fundamental limitation of single reward modeling in RLHF. A reward model trained to approximate an average preference inevitably suppresses minority or context-specific preferences, leading to misalignment when deployed in settings where user intent deviates from the dominant patterns in the training data. Such failures are particularly problematic in interactive or personalized systems, where alignment with individual users is critical.

To better capture preference heterogeneity, recent work has proposed training \emph{multiple} reward models, each corresponding to a predefined domain, annotator group, or objective \cite{ovadya2023generative,rame2023rewarded,chakraborty2024maxmin, park2024rlhf}. While this strategy improves expressivity within known settings, it introduces several new challenges. First, the set of reward models must be specified in advance, making it difficult to accommodate dynamic changes in human preferences over time. Second, when new preference data are collected from previously unseen humans or domains, a new reward model typically needs to be trained or existing models must be retrained \cite{park2024rlhf,jang2023personalized,singh2025fspo}. This process is computationally costly, requires repeated data collection, and does not scale well as the number of distinct preference types grows. As a result, multi-reward modeling remains brittle under \emph{task shift} and offers limited flexibility in open-ended, real-world deployments.

In this work, we explore a different paradigm motivated by the in-context learning capabilities of transformer-based models to overcome the above-mentioned issues of multi-reward modeling. In-context learning allows a pretrained transformer to adapt to new tasks at inference time by conditioning on a small number of examples, without requiring parameter updates \cite{min2022rethinking, bai2023transformers, zhang2024trained}. This property suggests a promising alternative for preference modeling, i.e., rather than training or retraining reward models whenever preferences change, a model could \emph{infer} an intrinsic reward function on the fly from in-context preference demonstrations, which we refer to as \emph{in-context reward adaptation}.

However, a central finding of this paper is that naively applying in-context learning to pairwise preference data is insufficient. We show that when only binary comparative labels are available, in-context adaptation to unseen human reward models is fundamentally impossible, even with expressive transformer architectures and unlimited data. To resolve this problem, we introduce human \emph{response time} as an auxiliary behavioral signal. Drawing inspiration from cognitive decision-making models \cite{wagenmakers2007ez,ratcliff2008diffusion,berlinghieri2023measuring}, we show that response time encodes information about the \emph{strength} of preferences, complementing the directional information provided by binary comparisons. We demonstrate theoretically that incorporating response time restores identifiability of reward parameters and enables correct in-context adaptation to previously unseen human preferences, which shows a promising approach for robust reward modeling under preference distribution shift. Experiments on both synthetic and real-world human decision-making datasets further validate our analysis, showing substantial improvements in robustness of our proposed in-context reward adaptation under preference distribution shift.

In general, our results highlight fundamental limitations of existing reward modeling of RLHF pipelines based on static reward modeling and binary preferences. More broadly, this work suggests that scalable and robust human–AI alignment requires both adaptive learning mechanisms and richer forms of human feedback beyond binary comparisons.

\subsection*{Related Work}
\paragraph{Reward modeling with heterogeneous preferences.}
Aligning AI models to match diverse human preference data has been comprehensively studied in recent years. \citet{chakraborty2024maxmin} shows a single reward model is insufficient to align with heterogeneous preferences, which inspires a min-max framework to train a compromised policy. \citet{park2024rlhf} proposes a multi-reward learning approach based on representation learning and clustering and a preference aggregation approach to learn a single reward. \citet{sorensen2024roadmap} considers a multi-objective reward modeling to balance several distinct objectives induced by diverse preferences. In \cite{singh2025fspo} a meta-learning approach is proposed to fit personalized reward models with few-shot preference demonstrations. Synthetic persona-guided methods are introduced in \cite{ryan2025synthesizeme,zhang2024guided} in order to adapt to heterogeneous preferences. These works only consider reward adaptation and personalization to existing human preference data, or to unseen preferences but with asymptotic biases, while our method allows zero-bias adaptation to any preference data either seen or unseen.

\paragraph{In-context learning.}
Transformer-based models have been observed to exhibit in-context learning abilities in natural language processing \cite{wei2022emergent,dasgupta2022language,zhang2022opt}. \citet{garg2022can} initiates the study of in-context learning under a mathematical framework, showing transformers can in-contextly learn linear regression, two-larger ReLU networks, decision trees, etc. A line of works focuses on the construction of transformers, showing that the constructed transformers can perform as gradient descent \cite{bai2023transformers} and second-order algorithms \cite{fu2024transformers,giannou2024well}. Another line of research focuses on training dynamics and optimization landscape of transformers, where \citet{zhang2024trained} and proves a single linear-attention layer can in-contextly learn linear regression and is robust to task shift. \citet{huang2023context} shows similar results by further consider the additional softmax module. Recently, \citet{shen2024training} shows the in-context learning ability of linear attention transformers to classification problems by assuming mixture-of-Gaussian posteriors of the feature space, while our results are not restricted to such assumptions on the feature space.

\section{In-Context Learning Framework for RLHF with Heterogeneous Preferences}
\label{sec:icl-rlhf}

We study reward modeling of Reinforcement Learning from Human Feedback (RLHF) in a setting where human preferences are \emph{heterogeneous} and drawn from an unknown population. Unlike classical formulations that assume a single latent reward function shared across all annotators, we explicitly model variability across persons and investigate whether a single in-context learner can \emph{adaptively infer} a new person's reward function from a small number of preference demonstrations.

\subsection{Human Preference Model}
\label{subsec:pref-model}

Let $x \in \mathcal{X}$ denote a prompt, and let $y_0, y_1 \in \mathcal{Y}(x)$ be two candidate responses generated by a pretrained language model. We assume the existence of a population of human types\footnote{The human type means a group of humans sharing the same reward model.} indexed by $i$, each characterized by an intrinsic reward function
$
r_i : \mathcal{X} \times \mathcal{Y} \to \mathbb{R}.
$
For a given human type $i$, preferences between candidate responses follow a Bradley--Terry (BT) model \cite{bradley1952rank}:
\begin{equation}
\label{eq:BT}
p_i^*(y_w \succ y_l \mid x)
= \sigma\!\left( r_i(x,y_w) - r_i(x,y_l) \right),
\end{equation}
where $\sigma(t) = (1 + e^{-t})^{-1}$ is the sigmoid function, and $y_w, y_l$ denote the preferred and unpreferred responses, respectively.

Throughout the paper, we adopt a linear reward parameterization
\begin{equation}
\label{eq:linear-reward}
r_i(x,y) = \phi(x,y)^\top \theta_i^*,   \nonumber
\end{equation}
where $\phi(x,y) \in \mathbb{R}^d$ is a shared, known feature representation, and $\theta_i^* \in \mathbb{R}^d$ is an unknown parameter vector associated with human type $i$. The parameters $\theta_i^*$ are drawn from an unknown population distribution $\mathcal{H}$, which may have continuous support.

For a pairwise comparison $(x, y_0, y_1)$, define a binary preference variable
\[
z_i =
\begin{cases}
+1, & \text{if } y_1 \succ y_0, \\
-1, & \text{if } y_0 \succ y_1 .
\end{cases}
\]
Then the BT model can be equivalently written as
\begin{equation}
\label{eq:BT-z}
\mathbb{P}(z_i = 1 \mid x, y_0, y_1)
= \sigma\!\left( \big(\phi(x,y_1) - \phi(x,y_0)\big)^\top \theta_i^* \right).
\end{equation}
For notational convenience, we define the \emph{difference feature}
\[
\tilde{\phi}(x,y_0,y_1) := \phi(x,y_1) - \phi(x,y_0)
\]
and hence
$$
    \mathbb{P}(z_i=1 \mid x,y_0,y_1) = \sigma \left(\tilde{\phi}(x,y_0, y_1)^T \theta_i^* \right).
$$

\subsection{Limits of Multi-Reward Modeling}
\label{subsec:motivation}

A common strategy for handling heterogeneous preferences is to train \emph{multiple} reward models, each corresponding to a known human type or domain. While effective in-distribution, such approaches fundamentally assume access to a \emph{fixed} and \emph{finite} set of preference distributions during training \cite{park2024rlhf,singh2025fspo}. In practice, however, the population distribution $\mathcal{H}$ may evolve over time due to changes in cultural context, task requirements, expertise, or prolonged interaction with the system. Or there might be new human annotators exhibit preferences that are totally unseen when training the multi-reward models. These circumstances introduce parameters $\theta_{\mathrm{new}} \notin \mathcal{H}$ in practice. In such cases, static reward models suffer from \emph{out-of-distribution (OOD) degradation}, and performance can only be recovered by retraining or fine-tuning on new preference data, which is costly, slow, and incompatible with rapid human-in-the-loop alignment.

This motivates a different paradigm, i.e., instead of learning a fixed mapping from $(x,y)$ to reward, we seek a model that is \emph{robust} to OOD degradation $\theta_{new} \in \mathcal{H}$ and can adapt its effective reward function in context, conditioned on a small set of observed preferences.

\subsection{In-Context Reward Adaptation via Transformers}
\label{subsec:icl-transformer}

We leverage the in-context learning capabilities of transformers to perform on-the-fly inference of human reward parameters from preference demonstrations.

\paragraph{Preference demonstrations.}
Fix a human type $i$. Consider a set of $N$ prompt--response pairs
$
\{(x_i^l, y_{i,0}^l, y_{i,1}^l)\}_{l=1}^N,
$
for which preference data are collected. We additionally consider a \emph{query} instance $(x_i^q, y_{i,0}^q, y_{i,1}^q)$, for which the goal is to predict the preference of the same human type $i$ using only the in-context demonstrations.

\paragraph{Prompt matrix construction.}
We encode the demonstration data into a prompt matrix
\[
E_i =
\begin{bmatrix}
\phi_{i,0}^1 & \phi_{i,0}^2 & \cdots & \phi_{i,0}^N & \phi_{i,0}^q \\
\phi_{i,1}^1 & \phi_{i,1}^2 & \cdots & \phi_{i,1}^N & \phi_{i,1}^q \\
z_i^1   & z_i^2   & \cdots & z_i^N   & *
\end{bmatrix},
\]
where $\phi_{i,0}^l := \phi(x_i^l, y_{i,0}^l)$, $\phi_{i,1}^l := \phi(x_i^l, y_{i,1}^l)$, and the final entry $*$ corresponds to the unknown label for the query pair. To isolate preference-relevant structure, we apply a fixed linear transformation
\[
\tilde{E}_i =
\begin{bmatrix}
 -I_d & I_d & 0 \\
 0_d^\top & 0_d^\top & 1
\end{bmatrix}
E_i,
\]
so that each column of $\tilde{E}_i$ contains the difference feature $\tilde{\phi}_i^l = \phi_{i,1}^l - \phi_{i,0}^l$ together with its associated preference signal, i.e., after transformation
\[
\tilde{E}_i =
\begin{bmatrix}
\tilde{\phi}_{i}^1 & \tilde{\phi}_{i}^2 & \cdots & \tilde{\phi}_{i}^N & \tilde{\phi}_{i}^q \\
z_i^1   & z_i^2   & \cdots & z_i^N   & *
\end{bmatrix}.
\]

\paragraph{Transformer architecture.}
Following the linear-attention transformer formulation of \cite{huang2023context,zhang2024trained,shen2024training}, we consider a single-layer transformer of the form
\begin{equation}
\label{eq:transformer}
F(W^V, W^{KQ}; \tilde{E}_i)
= \tilde{E}_i + W^V \tilde{E}_i
\frac{\tilde{E}_i^\top W^{KQ} \tilde{E}_i}{N}.
\end{equation}
Similar to \cite{wu2023many,ahn2023transformers} we restrict the parameter matrices to
\[
W^V =
\begin{bmatrix}
0_{d \times d} & 0_d \\
0_d^\top & 1
\end{bmatrix},
\qquad
W^{KQ} =
\begin{bmatrix}
U & 0_d \\
0_d^\top & 0
\end{bmatrix},
\]
with $U \in \mathbb{R}^{d \times d}$ trainable. Note that this structure simplifies the our theoretical analysis while preserving the critical properties of the transformer.

\paragraph{Induced prediction.}
Under this architecture, the predicted preference for the query instance takes the form
\begin{equation}
\label{eq:prediction}
\mathbb{P}(\hat{z}_i^q = 1 \mid x^q, y_0^q, y_1^q)
=
\sigma\!\left(
\frac{1}{N}
\sum_{l=1}^N
z_i^l \,
(\tilde{\phi}^l)^\top
U
\tilde{\phi}^q
\right),
\end{equation}
where $\tilde{\phi}^q = \phi(x^q, y_1^q) - \phi(x^q, y_0^q)$. We set $\hat{z}_i^q = 1$ if the probability in \eqref{eq:prediction} exceeds $1/2$, and $\hat{z}_i^q = -1$ otherwise.

\subsection{Training Procedure}
\label{subsec:training}

We assume that difference features are drawn from a distribution $\tilde{\phi} \sim P_{\tilde{\phi}}$, independently of $\theta_i^* \sim \mathcal{H}$. Following \cite{zhang2024trained,shen2024training}, we evaluate performance in expectation over both human types and query instances, and minimize the expected cross-entropy loss
\begin{align}
\label{eq:LN}
L_N(U)
&=
-\frac{1}{2}
\mathbb{E}_{\theta_i^* \sim \mathcal{H}, \tilde{\phi} \sim P_{\tilde{\phi}}}
\Big[
(1 + z_i^q)\log \mathbb{P}(\hat{z}_i^q = 1) 
+
(1 - z_i^q)\log \mathbb{P}(\hat{z}_i^q = -1)
\Big],
\end{align}
where $z_i^q$ is generated according to the ground-truth BT model \eqref{eq:BT-z}, and $\mathbb{P}(\hat{z}_i^q = \pm 1)$ is given by \eqref{eq:prediction}.

\section{Impossibility of In-Context Reward Adaptation from Binary Preferences}
\label{sec:impossibility}

In this section, we study the fundamental limits of in-context reward adaptation when the transformer is provided only with binary comparative preference labels. Despite the apparent flexibility of the in-context learning framework introduced in the previous section, we show that preference data alone are insufficient for robustly adapting to unseen human reward parameters. This negative result holds even under idealized conditions with infinite data and perfect optimization, revealing a structural obstruction rather than a finite-sample or algorithmic limitation.

Our analysis proceeds in two stages. We first characterize the training dynamics of the transformer and establish convergence to a well-defined population-level objective. We then analyze inference on a previously unseen human and show that asymptotically correct prediction is generically impossible.

For human type $i$ we define the preference moment
\[
\mu_i := \mathbb{E}_{z_i, \tilde{\phi}}\!\left[ z_i \tilde{\phi} \mid \theta_i^* \right].
\]

We first analyze the optimization landscape induced by the expected training objective. See Appendix \ref{apx:proof_thm-loss} for the proof.

\begin{theorem}[Asymptotic Optimality]
\label{thm_loss-convergence}
Assume the following conditions hold:
\begin{enumerate}
    \item $\mathbb{E}_{i \sim \mathcal{H}}(\mu_i \mu_i^\top)$ and $\mathbb{E}_{\tilde{\phi} \sim P_{\tilde{\phi}}}(\tilde{\phi} \tilde{\phi}^\top)$ are full rank;
    \item the feature difference satisfies $\|\tilde{\phi}\| \le B$ almost surely;
    \item optimization is restricted to the bounded parameter set
    \[
        \mathcal{U} = \left\{ U \in \mathbb{R}^{d \times d} \mid \|U\|_F \le R \right\}.
    \]
\end{enumerate}
Then the following statements hold:
\begin{enumerate}
    \item[i.] The training objective $L_N(U)$ \eqref{eq:LN} is strongly convex with respect to $U$.
    \item[ii.] The population objective
    \begin{align}
        \bar{L}(U)
        &:=
        - \frac{1}{2}
        \mathbb{E}
        \Big[
        (1+z_i^q)\log \sigma(\mu_i^\top U \tilde{\phi}_i^q)  \nonumber \\
        &~ +
        (1-z_i^q)\log \bigl(1-\sigma(\mu_i^\top U \tilde{\phi}_i^q)\bigr)
        \Big]
    \end{align}
    is also strongly convex.
    \item[iii.] The uniform convergence bound
    \[
        \bigl| L_N(U) - \bar{L}(U) \bigr| = \mathcal{O}(N^{-1}), 
        \quad \forall U \in \mathbb{R}^{d \times d},
    \]
    holds.
    \item[iv.] Let $U_N^*$ and $\bar{U}^*$ denote the unique minimizers of $L_N$ and $\bar{L}$, respectively. $U_N^*, \bar{U}^*$ are unique and
    \[
        \| U_N^* - \bar{U}^* \|_F = \mathcal{O}(N^{-1}).
    \]
\end{enumerate}
\end{theorem}

Theorem~\ref{thm_loss-convergence} shows that the training problem is well behaved: the objective is strongly convex, admits a unique minimizer, and empirical risk minimization consistently recovers the population optimum. Consequently, standard optimization methods converge efficiently, and any failure of in-context adaptation cannot be attributed to optimization instability or finite-sample effects.

\subsection{Inference on An Unseen Human}

We now analyze inference for a previously unseen human. Suppose the transformer has been trained to obtain $U_N^*$. We are given $M$ preference samples collected from a new human with reward parameter $\theta_{\mathrm{new}}$, without assuming $\theta_{\mathrm{new}} \sim \mathcal{H}$.

The transformer predicts the preference on a query instance $(x^q, y_0^q, y_1^q)$ according to
\[
\mathbb{P}(\hat{z}_{\mathrm{new}}^q = 1 \mid x^q, y_0^q, y_1^q)
=
\sigma\!\left(
\frac{1}{M}
\sum_{j=1}^M
z_{\mathrm{new}}^j
(\tilde{\phi}^j)^\top
U_N^*
\tilde{\phi}^q
\right),
\]
while the true preference probability is
\[
\mathbb{P}(z_{\mathrm{new}}^q = 1 \mid x^q, y_0^q, y_1^q)
=
\sigma\!\left( (\tilde{\phi}^q)^\top \theta_{\mathrm{new}} \right).
\]

Define the empirical preference moment
\[
\hat{\mu}_{\mathrm{new}} := \frac{1}{M} \sum_{j=1}^M z_{\mathrm{new}}^j \tilde{\phi}^j.
\]
By the law of large numbers,
\[
\hat{\mu}_{\mathrm{new}} \xrightarrow{p} \mu_{\mathrm{new}}
\quad \text{as } M \to \infty,
\]
where
\[
\mu_{\mathrm{new}}
=
\mathbb{E}_{\tilde{\phi}}
\Big[
\tanh\!\left(\tfrac{1}{2}\tilde{\phi}^\top \theta_{\mathrm{new}}\right)
\tilde{\phi}
\Big].
\]
Together with the convergence $U_N^* \to \bar{U}^*$ from Theorem~\ref{thm_loss-convergence}, we obtain
\[
\forall \tilde{\phi}^q,
\qquad
\hat{\mu}_{\mathrm{new}}^\top U_N^* \tilde{\phi}^q
\xrightarrow{p}
\mu_{\mathrm{new}}^\top \bar{U}^* \tilde{\phi}^q,
\quad \text{as } N,M \to \infty.
\]

Since the sigmoid function is strictly monotone, asymptotically correct prediction for all $\tilde{\phi}^q$ requires
\begin{equation}
\label{eq_cond-pred}
\theta_{\mathrm{new}} = (\bar{U}^*)^\top \mu_{\mathrm{new}},
\qquad \forall \theta_{\mathrm{new}}.
\end{equation}
However, $\mu_{\mathrm{new}}$ is a nonlinear function of $\theta_{\mathrm{new}}$, involving both the hyperbolic tangent and an expectation over the feature distribution. Condition~\eqref{eq_cond-pred} therefore requires a single linear operator $(\bar{U}^*)^\top$ to invert this nonlinear mapping for all possible reward parameters, which is generically impossible.

\subsection{Impossibility Result and Geometric Interpretation}

We formalize this intuition in the following theorem (see proof in Appendix \ref{apx:proof_thm-fail}).

\begin{theorem}[Failure of In-Context Reward Adaptation]
\label{thm_failure-ICL}
There exist distributions $P_{\tilde{\phi}}$ and $\mathcal{H}$ such that, for some $\theta_{\mathrm{new}} \in \operatorname{supp}(\mathcal{H})$,
\[
\lim_{N,M \to \infty}
d_{\mathrm{TV}}\!\left(
\hat{P}_{\mathrm{new}},
P_{\mathrm{new}}
\mid \tilde{\phi}^q
\right)
> 0,
\quad \forall \tilde{\phi}^q,
\]
where
\[
d_{\mathrm{TV}}\!\left(
\hat{P}_{\mathrm{new}},
P_{\mathrm{new}}
\mid \tilde{\phi}^q
\right)
=
\mathrm{TV}\!\left(
\mathbb{P}(\hat{z}_{\mathrm{new}}^q \mid \tilde{\phi}^q),
\mathbb{P}(z_{\mathrm{new}}^q \mid \tilde{\phi}^q)
\right).
\]
\end{theorem}

Theorem~\ref{thm_failure-ICL} shows that even with infinitely many in-context demonstrations and perfect optimization, the transformer fails to correctly predict preferences for certain unseen human reward parameters.

This impossibility result can be understood geometrically by examining how binary preference data embed reward parameters into the representation space used for in-context inference. Each human reward parameter $\theta \in \mathbb{R}^d$ induces, through the Bradley--Terry model and the feature distribution $P_{\tilde{\phi}}$, a mapping defined by
\[
\mu(\theta)
=
\mathbb{E}_{\tilde{\phi}}
\Big[
\tanh\!\left(\tfrac{1}{2}\tilde{\phi}^\top \theta\right)\tilde{\phi}
\Big].
\]
This mapping $\theta \mapsto \mu(\theta)$ summarizes all information about $\theta$ that can be recovered from infinitely many binary comparisons. Importantly, this mapping is generally nonlinear and compressive.

From a geometric perspective, as $\theta$ varies over $\mathbb{R}^d$, the vectors $\mu(\theta)$ do not fill the $\mathbb{R}^d$ space uniformly. Instead, they lie on a curved, typically low-dimensional manifold
$
\mathcal{M}
=
\left\{ \mu(\theta) : \theta \in \mathbb{R}^d \right\}
\subset \mathbb{R}^d.
$
Binary preference labels therefore embed the space of reward parameters into this nonlinear manifold, collapsing different directions of variation in $\theta$ into possibly indistinguishable representations.

The in-context transformer observes a new human only through an empirical estimate of $\mu(\theta)$ and applies a fixed linear decoding map $(\bar{U}^*)^\top$ learned during training. Perfect in-context adaptation would require this linear map to act as a global inverse of the nonlinear embedding $\theta \mapsto \mu(\theta)$, i.e.,
\[
(\bar{U}^*)^\top \mu(\theta) = \theta
\quad \text{for all } \theta.
\]
Geometrically, this would require a single linear operator to decode a distorted, nonlinear embedding over the entire reward parameter space. Such an inversion is generically impossible unless $\mu(\theta)$ is linear in $\theta$ or the set of admissible reward parameters is severely restricted.

% This obstruction can be illustrated even in simple settings. For example, when $\tilde{\phi} \sim \mathcal{N}(0, I_d)$, the preference moment takes the form
% \[
% \mu(\theta) = c(\|\theta\|)\,\theta
% \]
% for a nonlinear scalar function $c(\cdot)$. In this case, all reward parameters lying on the same ray are mapped to colinear points in representation space, differing only by a nonlinear scaling factor. No fixed linear map can recover $\theta$ from $\mu(\theta)$ across all magnitudes, demonstrating that the failure of in-context adaptation persists even under highly symmetric feature distributions.

Theorem~\ref{thm_failure-ICL} reveals a fundamental limitation of in-context learning for RLHF, i.e., although transformers can aggregate and generalize patterns across demonstrations, they are constrained by the statistical sufficiency of the signals they receive. Binary comparative feedback collapses continuous reward parameters into nonlinear summary statistics that cannot be universally decoded by a fixed linear transformation. As a result, standard in-context learning architectures cannot robustly adapt to heterogeneous or previously unseen human preferences using preference labels alone.

\paragraph{Motivation for introducing auxiliary signals.}
The geometric perspective also suggests a plausible solution to resolve the issue. The failure arises because binary comparisons encode only the \emph{sign} of $\tilde{\phi}^\top \theta$, discarding magnitude information. Auxiliary signals that correlate with the strength of preference can enrich the embedding $\theta \mapsto \mu(\theta)$ and restore identifiability. In the next section, we show that incorporating response time effectively linearizes the embedding by recovering information about $|\tilde{\phi}^\top \theta|$, thereby enabling robust in-context reward adaptation for unseen human preferences.

\section{In-Context Reward Adaptation with Response Time}
\label{sec:response-time}

As established in Theorem~\ref{thm_failure-ICL}, in-context learning based solely on binary preference labels cannot universally adapt to unseen human reward parameters. In this section, we show that incorporating human \emph{response time} as an auxiliary signal resolves this obstruction and enables correct in-context reward adaptation.

We model human decision-making using a standard drift--diffusion process following \cite{li2024enhancing,berlinghieri2023measuring,wagenmakers2007ez}. For a human of type $i$ comparing responses $y_0$ and $y_1$ given prompt $x$, the response time is defined as
\[
t_i(x,y_0,y_1)
:=
\inf\bigl\{
\tau > 0 \;\big|\;
S_i(\tau) \in \{-\tfrac{1}{2}, \tfrac{1}{2}\}
\bigr\},
\]
where the latent decision variable evolves according to
\[
S_i(\tau)
=
\bigl(r_i(x,y_1) - r_i(x,y_0)\bigr)\tau + B(\tau),
\]
where $B(\tau)$ is a standard Brownian motion.
Here $\pm \tfrac{1}{2}$ denote absorbing decision boundaries\footnote{In \cite{berlinghieri2023measuring,li2024enhancing} absorbing decision boundary level is set by $a$ instead of $1/2$ considered here. However, we note that we can always incorporate $2a$ into $\theta_i^*$, forming a new $\tilde{\theta_i^*}$ which is of interest, to make the setting consistent with ours.}. 
% Under this model, the response time encodes not only the direction of preference but also its strength, i.e., $\vert \tilde{\phi}^T \theta_i^* \vert$ as shown in the following.

Following \cite{palmer2005effect}, a direct calculation yields the conditional expectation
\[
\mathbb{E}[t_i \mid \tilde{\phi}]
=
\begin{cases}
\dfrac{1}{2\,\tilde{\phi}^\top \theta_i^*}
\tanh\!\left(\dfrac{1}{2}\tilde{\phi}^\top \theta_i^*\right),
& \tilde{\phi}^\top \theta_i^* \neq 0, \\[8pt]
\dfrac{1}{4},
& \tilde{\phi}^\top \theta_i^* = 0.
\end{cases}
\]
Recalling that
$
\mathbb{E}[z_i \mid \tilde{\phi}]
=
\tanh\!\left(\tfrac{1}{2}\tilde{\phi}^\top \theta_i^*\right),
$
we obtain the key identity
\begin{equation}
\label{eq_reg-rela}
\tilde{\phi}^\top \theta_i^*
=
\frac{1}{2}
\frac{\mathbb{E}[z_i \mid \tilde{\phi}]}{\mathbb{E}[t_i \mid \tilde{\phi}]}.
\end{equation}

Equation~\eqref{eq_reg-rela} reveals the central insight of this section: combining preference labels with response time recovers a linear signal in the reward parameter. In contrast to binary labels alone, which encode only the sign of $\tilde{\phi}^\top \theta_i^*$, the ratio $z_i / t_i$ provides a continuous, magnitude-sensitive measurement (i.e., $\tilde{\phi}^T \theta_i^*$) that eliminates the geometric ambiguity discussed in Section~\ref{sec:impossibility}.

\paragraph{Prompt construction with response time.}
Assume that for each human type $i$ there are $K$ independent annotators sharing the same reward parameter $\theta_i^*$. We construct the augmented prompt
\[
E_{i,k}
=
\begin{bmatrix}
\phi_{i,0}^1 & \phi_{i,0}^2 & \cdots & \phi_{i,0}^N & \phi_{i,0}^q \\
\phi_{i,1}^1 & \phi_{i,1}^2 & \cdots & \phi_{i,1}^N & \phi_{i,1}^q \\
t_{i,k}^1 & t_{i,k}^2 & \cdots & t_{i,k}^N & * \\
z_{i,k}^1 & z_{i,k}^2 & \cdots & z_{i,k}^N & *
\end{bmatrix},
\]
where $(i,k)$ indexes the $k$-th annotator of type $i$. Aggregating over annotators yields
\[
z_i^l := \frac{1}{K}\sum_{k=1}^K z_{i,k}^l,
\qquad
t_i^l := \frac{1}{K}\sum_{k=1}^K t_{i,k}^l,
\]
and the averaged prompt
\[
E_i
=
\begin{bmatrix}
\phi_{i,0}^1 & \phi_{i,0}^2 & \cdots & \phi_{i,0}^N & \phi_{i,0}^q \\
\phi_{i,1}^1 & \phi_{i,1}^2 & \cdots & \phi_{i,1}^N & \phi_{i,1}^q \\
t_i^1 & t_i^2 & \cdots & t_i^N & * \\
z_i^1 & z_i^2 & \cdots & z_i^N & *
\end{bmatrix}.
\]

We construct the prompt matrix using the feature difference and ratio $z_i/ t_i$ as follows
\[
\tilde{E}_i
=
\begin{bmatrix}
\tilde{\phi}_i^1 & \tilde{\phi}_i^2 & \cdots & \tilde{\phi}_i^N & \tilde{\phi}_i^q \\
\dfrac{z_i^1}{t_i^1} & \dfrac{z_i^2}{t_i^2} & \cdots & \dfrac{z_i^N}{t_i^N} & *
\end{bmatrix},
\]
where $\tilde{\phi}_i^l = \phi_{i,1}^l - \phi_{i,0}^l.$

\paragraph{Training objective and its asymptotic behavior.}
Feeding $\tilde{E}_i$ into the same linear-attention transformer architecture yields the prediction
\[
\hat{o}_i^q
=
\frac{1}{N}
\sum_{l=1}^N
\frac{z_i^l}{t_i^l}
(\tilde{\phi}_i^l)^\top
U
\tilde{\phi}_i^q.
\]
We train the model using the squared regression loss in expectation

\begin{align}   \label{eq_poploss-reg}
    L_{N,K}(U) = \frac{1}{2}\mathbb{E}_{\theta_i^*\sim \mathcal{H},\tilde{\phi}_i^{1:N}, \tilde{\phi}_i^q \sim P_{\tilde{\phi}}, z_i^{1:N}, t_i^{1:N}}\left(\hat{o}_i^q - (z_i^q / t_i^q)  \right)^2
\end{align}

Defining $\hat{s}_i = \frac{1}{N}\sum_{l=1}^N \frac{z_i^l}{t_i^l} \tilde{\phi}_i^l$, then $\hat{o}_i^q = \hat{s}_i^T U \tilde{\phi}^q_i$. Conditioned on $\theta_i^*$, according to the Law of Large Numbers (and noting i.i.d. $\tilde{\phi}_i^l$ and independence from $\theta_i^*$),
\begin{align*}
    \hat{s}_i \overset{p}{\to} \mathbb{E}_{\tilde{\phi}}[2 \tilde{\phi}^T \theta_i^* \tilde{\phi}] = 2 \Sigma_{\tilde{\phi}} \theta_i^* , ~~\text{as}~ N, K \to \infty
\end{align*}
where $\Sigma_{\tilde{\phi}} = \mathbb{E}[\tilde{\phi}\tilde{\phi}^T]$. Also, for the query, 
$$
    \frac{z_i^q}{t_i^q} \overset{p}{\to} 2 (\tilde{\phi}_i^q)^T \theta_i^* ~~\text{as} ~ K \to \infty.
$$
Thus, for large $N$ and $K$, $L_{N,K}$ converges to $L_{\infty}$, where
\begin{align}   \label{eq_poploss-inf}
     L_{\infty}(U) &:=\lim_{N,K \to \infty} L_{N,K}(U) \nonumber \\
    &= \frac{1}{2}\mathbb{E}\left[ ((2\Sigma_{\tilde{\phi}} \theta_i^*)^T U \tilde{\phi}^q - 2(\tilde{\phi}^q)^T \theta_i^*)^2 \right] \nonumber \\
    &= 2\mathbb{E}\left[((\theta_i^*)^T (\Sigma_{\tilde{\phi}}U - I) \tilde{\phi}^q)^2 \right]
\end{align}
One can easily observe that a minimizer of $L_{\infty}$ is given by $U^* = \Sigma_{\tilde{\phi}}^{-1}$ (which is in fact also unique as shown in Theorem) if assuming $\Sigma_{\tilde{\phi}}$ is full-rank, which should match the minimizer of $L_{N,K}$ as $N, K \to \infty$. This is formally presented in the following theorem (see Appendix \ref{apx:proof_thm-res} for proof).

\begin{theorem} \label{thm_res-time-conv}
    Consider $L_{N,K}(U)$ and $L_{\infty}(U)$ defined by \eqref{eq_poploss-reg} and \eqref{eq_poploss-inf}, respectively. Assuming that $\Sigma_{\tilde{\phi}}:=\mathbb{E}[\tilde{\phi} \tilde{\phi}^T]$ and $\mathbb{E}[\theta_i^* (\theta_i^*)^T]$ are full-rank; $\Vert \tilde{\phi} \Vert \le B$, $\Vert U \Vert_F \le R$, then the following hold:
    \item[(i).] $L_{N,K}$ and $L_{\infty}$ are strongly-convex;

    \item[(ii).] $U^* = \Sigma_{\tilde{\phi}}^{-1}$ is the unique minimizer of $L_{\infty}$;

    \item[(iii).] $\Vert U_{N,K}^* - U^* \Vert_F \le \mathcal{O}(1/\sqrt{N} + 1/\sqrt{K})$, where $U_{N,K}^*$ is the unique minimizer of $L_{N,K}$.
\end{theorem}

\paragraph{Correct in-context adaptation for unseen humans.}
Now let us consider a new human type with arbitrary $\theta_{new}$ which is not necessarily from $\mathcal{H}$. Given $M$ in-context samples drawn from that human type with $K$ human users, as $U_{N,K}^* \to U^* = \Sigma_{\tilde{\phi}}^{-1}$ with $N,K \to \infty$, 
\begin{align}
    \hat{o}_{new}^q \overset{p}{\to} (2\Sigma_{\tilde{\phi}} \theta_{new})^T \Sigma_{\tilde{\phi}}^{-1} \tilde{\phi}^q = 2 \theta_{new}^T \tilde{\phi}^q \nonumber
\end{align}
which matches the true target as $z_{new}^q / t_{new}^q \overset{p}{\to} 2 \theta_{new}^T \tilde{\phi}^q, K \to \infty$. Therefore, the transformer is able to in-contextly adapt to human $\theta_{new}$'s preference model without retraining. The above discussion is formally characterized by the following statement (see Appendix \ref{apx:proof_coro} for proof).

\begin{corollary}[Correct In-Context Adaptation]   \label{coro_correct-pred}
    Given arbitrary $\theta_{new}$, let $\hat{o}_{new}^q$ be the prediction of the transformer after complete training, i.e., with $U_{N,K}^*$ being the minimizer of $L_{N,K}$,
    $$
        \hat{o}_{new}^q := \frac{1}{M} \sum_{l=1}^{M} \frac{z_{new}^l}{t_{new}^l} (\tilde{\phi}^l)^T U_{N,K}^* \tilde{\phi}^q
    $$
    where $z_{new}^l = (1/K)\sum_{k=1}^K z_{new, k}^l$, $ t_{new}^l = (1/K)\sum_{k=1}^K t_{new, k}^l$ are averages over $K$ humans drawn from the human type $\theta_{new}$. Then, 
    $$
        \mathbb{E}(\hat{o}_{new}^q - 2 \theta_{new}^T \tilde{\phi}^q)^2 \le \mathcal{O}(M^{-1} + N^{-1} + K^{-1} + (NK)^{-1/2}).
    $$
\end{corollary}

In summary, these results show that response time resolves the limitation of in-context learning with binary preference data in Section~\ref{sec:impossibility}. By enriching binary comparisons with a continuous signal proportional to preference strength, the transformer can recover a linear embedding of reward parameters and achieve correct in-context adaptation to unseen human preferences.

\section{Experiments}
We evaluate the proposed in-context reward adaptation framework on both a controlled synthetic dataset and a real-world human preference dataset with recorded response times. The synthetic experiments are designed to closely mirror the theoretical setting and allow precise validation under known ground-truth reward models, while the real-world experiments demonstrate the applicability of response-time–augmented in-context learning on behavioral data collected from human subjects.

\subsection{Synthetic Dataset}
For the synthetic setting, we construct a preference learning environment that follows the setting of our problem. For each query, we independently sample two feature vectors $\phi_0, \phi_1 \in \mathbb{R}^d$ from some fixed distributions, and define the feature difference $\tilde{\phi} := \phi_1 - \phi_0$. 

Human reward parameters $\theta_i^*$ are sampled from a mixture of two Gaussian distributions, inducing heterogeneity across human types. To in-context learning ability to unseen preferences, we additionally sample a new human type $\theta_{\mathrm{new}}$ from a third Gaussian distribution that is disjoint from the training mixture, thereby creating an explicit out-of-distribution (OOD) test setting. To obtain response times, we simulate a drift--diffusion process as described in Section~\ref{sec:response-time}. Specifically, for each human type, we simulate $K$ independent annotators by running the drift--diffusion process $K$ times per query.

Although our theoretical results focus on a linear-attention transformer for analytical tractability, we also implement the in-context learning framework using a GPT-2 model with $124$M parameters. Both linear attention model and GPT-2 model are trained on the same synthesized preference data, using either binary preference labels alone or the response-time–augmented targets introduced in Section~\ref{sec:response-time}. 

\begin{table}[h]
\centering
\caption{Test accuracy under in-distribution (ID) and out-of-distribution (OOD) preference data for linear attention and GPT2}

\quad

\begin{tabular}{p{8em}p{6em}p{6em}}
    \toprule
     Setting & LinearAttn & GPT2 \\
    \midrule
    w/o resp (ID)  & 0.936 &  0.925 \\
    w/o resp (OOD) & 0.783 &  0.694 \\
    \midrule
    w/ resp (ID)   & 0.878 &  0.905 \\
    w/ resp (OOD)  & 0.891 &  0.875 \\
    \bottomrule
\end{tabular}
\label{table:exp_syn}
\end{table}

Numerical results are presented in Table \ref{table:exp_syn}, where "resp" represents "response time"; "ID" or "OOD" means that the new human parameter is drawn from the same distribution as used for training or from a different distribution. When trained without response time, both linear attention and GPT-2 perform well on in-distribution preference data but exhibit a pronounced degradation under OOD preferences. This behavior is consistent with our theoretical analysis, which predicts that binary preference labels alone are insufficient for in-context adaptation to unseen reward parameters.

Incorporating response time substantially improves OOD performance for both architectures, bringing test accuracy close to in-distribution levels. This supports the central claim of the paper: response time resolves the information bottleneck inherent in binary comparisons and enables robust in-context reward adaptation to unseen human preferences. Notably, the same qualitative trend appears for both linear attention and GPT-2, indicating that the failure of binary-only in-context learning and the benefits of response time are not artifacts of limited model capacity, but stem from fundamental properties of preference information.

\subsection{Real-World Dataset: Food-Risk Preferences}

To evaluate our approach on real human data, we adopt the food-risk dataset introduced by \cite{smith2018attention} and used in prior work.

The dataset consists of binary choices and response times collected from $42$ participants, each responding to between $60$ and $200$ queries. Each query presents two arms, where each arm contains two food items. By selecting an arm, participants receive one of the two food items uniformly at random. In addition to choices and response times, participants’ eye movements were recorded during the experiment.

\begin{table}[h]
\caption{Test accuracy under in-distribution (ID) and out-of-distribution (OOD) food-risk preference data for GPT2 across different inference lengths}

\quad

\centering
\begin{tabular}{p{8em}p{4em}p{4em}p{4em}}
    \toprule
     Setting       & $M=4$ & $M=8$    &  $M=16$ \\
    \midrule
    w/o resp (ID)  & 0.606  &  0.679  &  0.675  \\
    w/o resp (OOD) & 0.581  &  0.625  &  0.631  \\
    \midrule
    w/ resp (ID)   & 0.633  &  0.681  &  0.710  \\
    w/ resp (OOD)  & 0.605  &  0.667  &  0.705  \\
    \bottomrule
\end{tabular}
\label{table:exp_real}
\end{table}

\begin{figure}[h]
    \centering
    \includegraphics[width=0.6\linewidth]{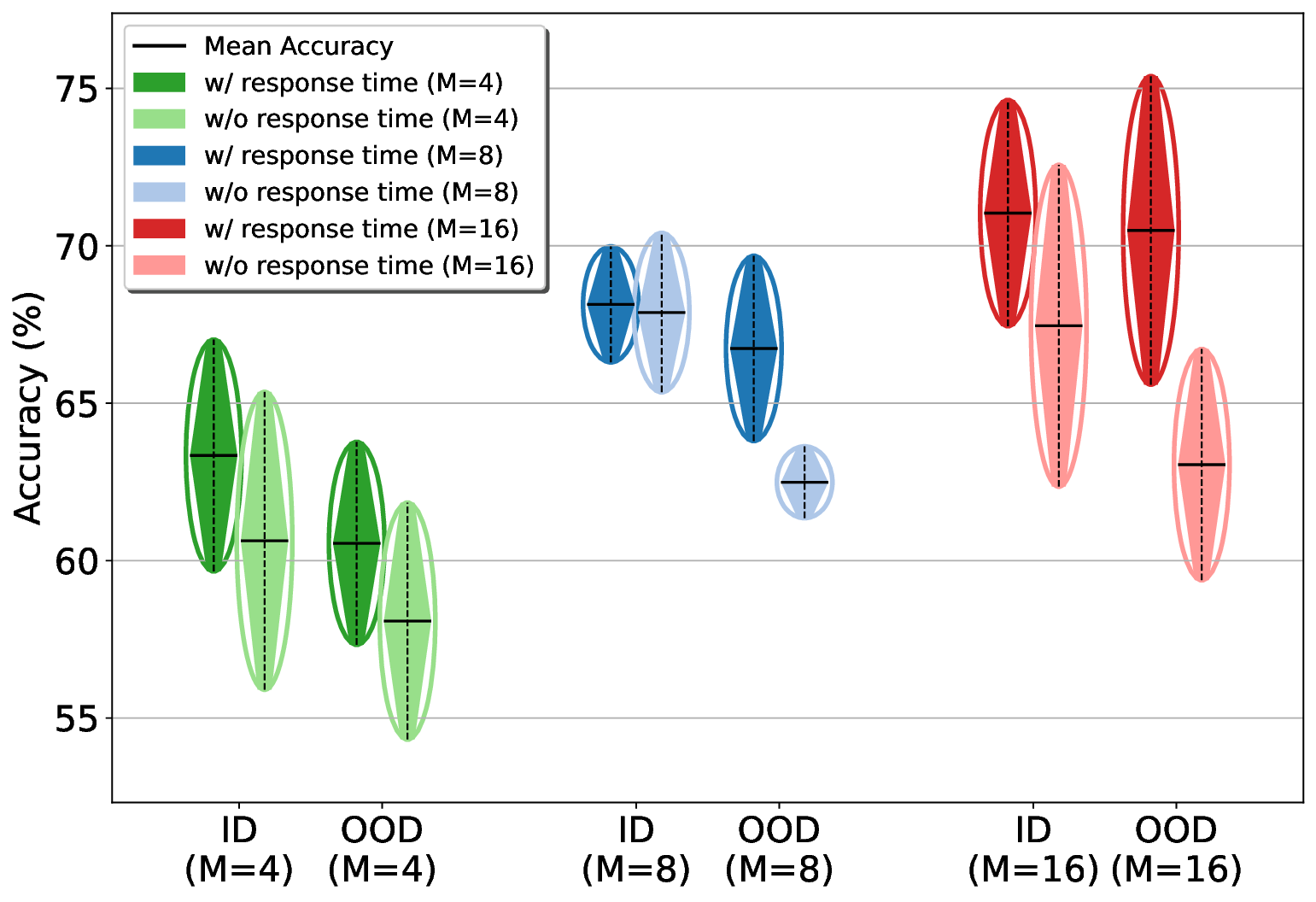}
    \caption{Inference accuracy (mean $\pm$ std) across different $M$}
    \label{fig:exp_real}
\end{figure}

Prior to the choice task, participants rated all food items on a discrete scale from $-10$ to $10$. Following prior work \cite{li2024enhancing}, we define each arm’s feature vector as the participant’s ratings of the two food items in that arm, augmented with second-order polynomial features. This yields a feature representation in $\mathbb{R}^5$.

Table~\ref{table:exp_real} and Figure \ref{fig:exp_real} show performance on the food-risk dataset when training with GPT-2, where a small subset of participants is grouped and treated as a new human type, and the remaining participants are used for training. We vary the inference length $M$, corresponding to the number of in-context preference demonstrations available at test time. Similar to the results of synthetic data, incorporating response time generally narrows the performance gap between in-distribution and out-of-distribution inferences, indicating effective in-context adaptation to unseen human preferences. This phenomenon becomes more significant and clear as inference length $M$ increases, which is consistent with our theory, as longer demonstration sequences improve in-context learning ability of the trained transformer.

\section{Limitations}
We briefly discuss the limitations of this work. For theoretical tractability, our analysis focuses on a linear-attention transformer, which is a simplified abstraction of the architectures commonly used in practice. Although our experiments with GPT-2 demonstrate that the observed phenomena extend to standard transformer models, extending the theoretical guarantees to more complex architectures remains an important problem for future work. In addition, while response time serves as an effective auxiliary signal for resolving the failure of in-context adaptation under binary preference labels, it may be difficult to reliably measure or collect in practical settings. Exploring alternative auxiliary signals that are easier to obtain yet similarly informative for preference modeling is therefore a promising direction.

\section{Conclusion}
In this paper, We study in-context reward adaptation for preference modeling in RLHF and showed that in-context learning based only on binary preference labels is fundamentally insufficient for adapting to unseen human preferences. This limitation stems from an intrinsic information bottleneck, where binary labels collapse reward parameters into nonlinear representations that cannot be universally decoded by linear attentions. We address this issue by incorporating human response time as an auxiliary signal. Our theoretical analysis shows that response time recovers a linear signal of preference strength and enables correct in-context adaptation to unseen reward models. Experiments on both synthetic and real-world datasets support our theory, demonstrating improved robustness to distribution shift and more effective use of in-context demonstrations.

\nocite{langley00}

\bibliography{main}

@article{chakraborty2024maxmin,
  title={MaxMin-RLHF: Alignment with diverse human preferences},
  author={Chakraborty, Souradip and Qiu, Jiahao and Yuan, Hui and Koppel, Alec and Huang, Furong and Manocha, Dinesh and Bedi, Amrit Singh and Wang, Mengdi},
  journal={arXiv preprint arXiv:2402.08925},
  year={2024}
}

@article{park2024rlhf,
  title={Rlhf from heterogeneous feedback via personalization and preference aggregation},
  author={Park, Chanwoo and Liu, Mingyang and Kong, Dingwen and Zhang, Kaiqing and Ozdaglar, Asuman},
  journal={arXiv preprint arXiv:2405.00254},
  year={2024}
}

@article{li2024enhancing,
  title={Enhancing Preference-based Linear Bandits via Human Response Time},
  author={Li, Shen and Zhang, Yuyang and Ren, Zhaolin and Liang, Claire and Li, Na and Shah, Julie A},
  journal={Advances in Neural Information Processing Systems},
  volume={37},
  pages={16852--16893},
  year={2024}
}

@article{zhang2024trained,
  title={Trained transformers learn linear models in-context},
  author={Zhang, Ruiqi and Frei, Spencer and Bartlett, Peter L},
  journal={Journal of Machine Learning Research},
  volume={25},
  number={49},
  pages={1--55},
  year={2024}
}

@article{bai2023transformers,
  title={Transformers as statisticians: Provable in-context learning with in-context algorithm selection},
  author={Bai, Yu and Chen, Fan and Wang, Huan and Xiong, Caiming and Mei, Song},
  journal={Advances in neural information processing systems},
  volume={36},
  pages={57125--57211},
  year={2023}
}

@article{min2022rethinking,
  title={Rethinking the role of demonstrations: What makes in-context learning work?},
  author={Min, Sewon and Lyu, Xinxi and Holtzman, Ari and Artetxe, Mikel and Lewis, Mike and Hajishirzi, Hannaneh and Zettlemoyer, Luke},
  journal={arXiv preprint arXiv:2202.12837},
  year={2022}
}

@article{ouyang2022training,
  title={Training language models to follow instructions with human feedback},
  author={Ouyang, Long and Wu, Jeffrey and Jiang, Xu and Almeida, Diogo and Wainwright, Carroll and Mishkin, Pamela and Zhang, Chong and Agarwal, Sandhini and Slama, Katarina and Ray, Alex and others},
  journal={Advances in neural information processing systems},
  volume={35},
  pages={27730--27744},
  year={2022}
}

@inproceedings{korbak2023pretraining,
  title={Pretraining language models with human preferences},
  author={Korbak, Tomasz and Shi, Kejian and Chen, Angelica and Bhalerao, Rasika Vinayak and Buckley, Christopher and Phang, Jason and Bowman, Samuel R and Perez, Ethan},
  booktitle={International Conference on Machine Learning},
  pages={17506--17533},
  year={2023},
  organization={PMLR}
}

@article{berlinghieri2023measuring,
  title={Measuring utility with diffusion models},
  author={Berlinghieri, Renato and Krajbich, Ian and Maccheroni, Fabio and Marinacci, Massimo and Pirazzini, Marco},
  journal={Science Advances},
  volume={9},
  number={34},
  pages={eadf1665},
  year={2023},
  publisher={American Association for the Advancement of Science}
}

@article{wagenmakers2007ez,
  title={An EZ-diffusion model for response time and accuracy},
  author={Wagenmakers, Eric-Jan and Van Der Maas, Han LJ and Grasman, Raoul PPP},
  journal={Psychonomic bulletin \& review},
  volume={14},
  number={1},
  pages={3--22},
  year={2007},
  publisher={Springer}
}

@article{palmer2005effect,
  title={The effect of stimulus strength on the speed and accuracy of a perceptual decision},
  author={Palmer, John and Huk, Alexander C and Shadlen, Michael N},
  journal={Journal of vision},
  volume={5},
  number={5},
  pages={1--1},
  year={2005},
  publisher={The Association for Research in Vision and Ophthalmology}
}

@article{smith2018attention,
  title={Attention and choice across domains.},
  author={Smith, Stephanie M and Krajbich, Ian},
  journal={Journal of Experimental Psychology: General},
  volume={147},
  number={12},
  pages={1810},
  year={2018},
  publisher={American Psychological Association}
}

@article{wu2023many,
  title={How Many Pretraining Tasks Are Needed for In-Context Learning of Linear Regression?},
  author={Wu, Jingfeng and Zou, Difan and Chen, Zixiang and Braverman, Vladimir and Gu, Quanquan and Bartlett, Peter L},
  journal={arXiv preprint arXiv:2310.08391},
  year={2023}
}

@article{ahn2023transformers,
  title={Transformers learn to implement preconditioned gradient descent for in-context learning},
  author={Ahn, Kwangjun and Cheng, Xiang and Daneshmand, Hadi and Sra, Suvrit},
  journal={Advances in Neural Information Processing Systems},
  volume={36},
  pages={45614--45650},
  year={2023}
}

@article{sorensen2024roadmap,
  title={A roadmap to pluralistic alignment},
  author={Sorensen, Taylor and Moore, Jared and Fisher, Jillian and Gordon, Mitchell and Mireshghallah, Niloofar and Rytting, Christopher Michael and Ye, Andre and Jiang, Liwei and Lu, Ximing and Dziri, Nouha and others},
  journal={arXiv preprint arXiv:2402.05070},
  year={2024}
}

@article{singh2025fspo,
  title={Fspo: Few-shot preference optimization of synthetic preference data in llms elicits effective personalization to real users},
  author={Singh, Anikait and Hsu, Sheryl and Hsu, Kyle and Mitchell, Eric and Ermon, Stefano and Hashimoto, Tatsunori and Sharma, Archit and Finn, Chelsea},
  journal={arXiv preprint arXiv:2502.19312},
  year={2025}
}

@inproceedings{ryan2025synthesizeme,
  title={SynthesizeMe! inducing persona-guided prompts for personalized reward models in LLMs},
  author={Ryan, Michael J and Shaikh, Omar and Bhagirath, Aditri and Frees, Daniel and Held, William Barr and Yang, Diyi},
  booktitle={Proceedings of the 63rd Annual Meeting of the Association for Computational Linguistics (Volume 1: Long Papers)},
  pages={8045--8078},
  year={2025}
}

@article{zhang2024guided,
  title={Guided profile generation improves personalization with llms},
  author={Zhang, Jiarui},
  journal={arXiv preprint arXiv:2409.13093},
  year={2024}
}

@article{wei2022emergent,
  title={Emergent abilities of large language models},
  author={Wei, Jason and Tay, Yi and Bommasani, Rishi and Raffel, Colin and Zoph, Barret and Borgeaud, Sebastian and Yogatama, Dani and Bosma, Maarten and Zhou, Denny and Metzler, Donald and others},
  journal={arXiv preprint arXiv:2206.07682},
  year={2022}
}

@article{dasgupta2022language,
  title={Language models show human-like content effects on reasoning tasks},
  author={Dasgupta, Ishita and Lampinen, Andrew K and Chan, Stephanie CY and Sheahan, Hannah R and Creswell, Antonia and Kumaran, Dharshan and McClelland, James L and Hill, Felix},
  journal={arXiv preprint arXiv:2207.07051},
  year={2022}
}

@article{zhang2022opt,
  title={Opt: Open pre-trained transformer language models},
  author={Zhang, Susan and Roller, Stephen and Goyal, Naman and Artetxe, Mikel and Chen, Moya and Chen, Shuohui and Dewan, Christopher and Diab, Mona and Li, Xian and Lin, Xi Victoria and others},
  journal={arXiv preprint arXiv:2205.01068},
  year={2022}
}

@article{garg2022can,
  title={What can transformers learn in-context? a case study of simple function classes},
  author={Garg, Shivam and Tsipras, Dimitris and Liang, Percy S and Valiant, Gregory},
  journal={Advances in neural information processing systems},
  volume={35},
  pages={30583--30598},
  year={2022}
}

@article{giannou2024well,
  title={How Well Can Transformers Emulate In-context Newton's Method?},
  author={Giannou, Angeliki and Yang, Liu and Wang, Tianhao and Papailiopoulos, Dimitris and Lee, Jason D},
  journal={arXiv preprint arXiv:2403.03183},
  year={2024}
}

@article{fu2024transformers,
  title={Transformers learn to achieve second-order convergence rates for in-context linear regression},
  author={Fu, Deqing and Chen, Tian-Qi and Jia, Robin and Sharan, Vatsal},
  journal={Advances in Neural Information Processing Systems},
  volume={37},
  pages={98675--98716},
  year={2024}
}

@article{huang2023context,
  title={In-context convergence of transformers},
  author={Huang, Yu and Cheng, Yuan and Liang, Yingbin},
  journal={arXiv preprint arXiv:2310.05249},
  year={2023}
}

@article{shen2024training,
  title={On the training convergence of transformers for in-context classification of gaussian mixtures},
  author={Shen, Wei and Zhou, Ruida and Yang, Jing and Shen, Cong},
  journal={arXiv preprint arXiv:2410.11778},
  year={2024}
}

@article{wang2023aligning,
  title={Aligning large language models with human: A survey},
  author={Wang, Yufei and Zhong, Wanjun and Li, Liangyou and Mi, Fei and Zeng, Xingshan and Huang, Wenyong and Shang, Lifeng and Jiang, Xin and Liu, Qun},
  journal={arXiv preprint arXiv:2307.12966},
  year={2023}
}

@article{lambert2023alignment,
  title={The alignment ceiling: Objective mismatch in reinforcement learning from human feedback},
  author={Lambert, Nathan and Calandra, Roberto},
  journal={arXiv preprint arXiv:2311.00168},
  year={2023}
}

@article{shaikh2024aligning,
  title={Aligning Language Models with Demonstrated Feedback},
  author={Shaikh, Omar and Lam, Michelle S and Hejna, Joey and Shao, Yijia and Cho, Hyundong and Bernstein, Michael S and Yang, Diyi},
  journal={arXiv preprint arXiv:2406.00888},
  year={2024}
}

@article{casper2023open,
  title={Open problems and fundamental limitations of reinforcement learning from human feedback},
  author={Casper, Stephen and Davies, Xander and Shi, Claudia and Gilbert, Thomas Krendl and Scheurer, J{\'e}r{\'e}my and Rando, Javier and Freedman, Rachel and Korbak, Tomasz and Lindner, David and Freire, Pedro and others},
  journal={arXiv preprint arXiv:2307.15217},
  year={2023}
}

@article{ratcliff2008diffusion,
  title={The diffusion decision model: theory and data for two-choice decision tasks},
  author={Ratcliff, Roger and McKoon, Gail},
  journal={Neural computation},
  volume={20},
  number={4},
  pages={873--922},
  year={2008},
  publisher={MIT Press}
}

@article{ovadya2023generative,
  title={'Generative CI'through Collective Response Systems},
  author={Ovadya, Aviv},
  journal={arXiv preprint arXiv:2302.00672},
  year={2023}
}

@article{rame2023rewarded,
  title={Rewarded soups: towards pareto-optimal alignment by interpolating weights fine-tuned on diverse rewards},
  author={Rame, Alexandre and Couairon, Guillaume and Dancette, Corentin and Gaya, Jean-Baptiste and Shukor, Mustafa and Soulier, Laure and Cord, Matthieu},
  journal={Advances in Neural Information Processing Systems},
  volume={36},
  pages={71095--71134},
  year={2023}
}

@article{jang2023personalized,
  title={Personalized soups: Personalized large language model alignment via post-hoc parameter merging},
  author={Jang, Joel and Kim, Seungone and Lin, Bill Yuchen and Wang, Yizhong and Hessel, Jack and Zettlemoyer, Luke and Hajishirzi, Hannaneh and Choi, Yejin and Ammanabrolu, Prithviraj},
  journal={arXiv preprint arXiv:2310.11564},
  year={2023}
}

@article{bradley1952rank,
  title={Rank analysis of incomplete block designs: I. the method of paired comparisons},
  author={Bradley, Ralph Allan and Terry, Milton E},
  journal={Biometrika},
  volume={39},
  number={3/4},
  pages={324--345},
  year={1952},
  publisher={JSTOR}
}
\bibliographystyle{abbrvnat}

%%%%%%%%%%%%%%%%%%%%%%%%%%%%%%%%%%%%%%%%%%%%%%%%%%%%%%%%%%%%%%%%%%%%%%%%%%%%%%%
%%%%%%%%%%%%%%%%%%%%%%%%%%%%%%%%%%%%%%%%%%%%%%%%%%%%%%%%%%%%%%%%%%%%%%%%%%%%%%%
% APPENDIX
%%%%%%%%%%%%%%%%%%%%%%%%%%%%%%%%%%%%%%%%%%%%%%%%%%%%%%%%%%%%%%%%%%%%%%%%%%%%%%%
%%%%%%%%%%%%%%%%%%%%%%%%%%%%%%%%%%%%%%%%%%%%%%%%%%%%%%%%%%%%%%%%%%%%%%%%%%%%%%%
\newpage
\appendix
\onecolumn

\section{Proof of Theorem \ref{thm_loss-convergence}}   \label{apx:proof_thm-loss}

\begin{proof}
    To show Parts (i) and (ii),
    define $\hat{\mu}_i = \frac{1}{N}\sum_{l=1}^N z_i^l \tilde{\phi}_i^l$ and $\mu_i = \mathbb{E}[z_i \tilde{\phi} \mid \theta_i^*] = \mathbb{E}_{\tilde{\phi} \sim P_{\tilde{\phi}}}[\mathrm{tanh}(\frac{1}{2}\tilde{\phi}^T \theta_i^*)\tilde{\phi}]$. Let $\hat{s}_i(U) = \hat{\mu}_i^T U \tilde{\phi}^q_i$ and $s_i(U) = \mu_i^T U \tilde{\phi}^q_i$. Denote $w = \text{vec}(U) \in \mathbb{R}^{d^2}$ by vectorizing $U$. Then a simple calculation gives 
    $$
        \hat{s}_i(U) = \mathrm{tr}(U \tilde{\phi}^q_i \hat{\mu}_i^T) = \mathrm{tr}(U^T \hat{\mu}_i (\tilde{\phi}^q_i)^T) = w^T (\tilde{\phi}^q_i \otimes \hat{\mu}_i) .
    $$
    Denoting $l_N(w) = L_N(U)$ and letting $v = \tilde{\phi}^q_i \otimes \hat{\mu}_i$ and $y = (1 + z_i^q) / 2$, calculating its gradient and Hessian as follows
    $$
        \nabla l_N(w) = \mathbb{E}[(\sigma(w^T v) - y)v]
    $$
    $$
        \nabla^2 l_N(w) = \mathbb{E}[\sigma(w^T v) (1 - \sigma(w^T v))vv^T].
    $$
    Since $\Vert U \Vert_F \le R$, $\Vert \tilde{\phi} \Vert \le B, \text{with}~ \tilde{\phi} \sim P_{\tilde{\phi}}$ almost surely by assumptions, we obtain
    $$
        \vert w^T x \vert \le \Vert \hat{\mu}_i \Vert \Vert U \Vert_F \Vert \tilde{\phi}_q \Vert \le RB^2 < \infty
    $$
    and hence almost surely there exists some strictly positive constant $\lambda > 0$ such that
    $$
        \nabla^2 l(w) \ge \lambda \mathbb{E}[vv^T] \succ 0
    $$
    by noticing $\mathbb{E}[vv^T] = \mathbb{E}[\tilde{\phi}^q_i(\tilde{\phi}^q_i)^T] \otimes \mathbb{E}[\hat{\mu}_i \hat{\mu}_i^T] = \mathbb{E}[\tilde{\phi}^q_i(\tilde{\phi}^q_i)^T] \otimes \mathbb{E}[\mu_i \mu_i^T] \succ 0$ where the first equality follows the independence of $\tilde{\phi}^q_i$ and $\hat{\mu}_i$. This concludes the strong convexity of $L_N(U)$. Similarly, by replacing $\hat{\mu}_i$ by $\mu_i$, following the same procedure concludes the strong convexity of $\bar{L}(U)$.

    To show Part (iii), conditioned on $\theta_i^*$, $\tilde{\phi}^q_i$ and $z_i^q$, defining $g_U(u) = (1 + z_i^q) \log \sigma(u^T U \tilde{\phi}^q_i) + (1 - z_i^q) \log (1 - \sigma(u^T U \tilde{\phi}^q_i))$, then it is obvious that $L_N(U) = \mathbb{E}[g_U(\hat{\mu}_i)]$ and $\bar{L}(U) = \mathbb{E}[g_U(\mu_i)]$. For any fixed $U$, denoting $\delta = \hat{\mu}_i - \mu_i$, we have
    $$
        g_U(\hat{\mu}_i) - g(\mu_i) = \nabla g_U(\mu_i)\delta + \frac{1}{2} \delta^T \nabla^2 g_U(\mu_i + t \delta) \delta
    $$
    for some $t \in (0, 1)$. Taking expectations on both sides yields
    $$
        L_N(U) - \bar{L}(U) = \mathbb{E}\left[\frac{1}{2} \delta^T \nabla^2 g(\mu_i + t \delta) \delta \right]
    $$
    where we use the fact $\mathbb{E}[\delta \mid i] = 0$. Note that
    $$
        \nabla^2_u g_U(u) = \sigma(u^T U \tilde{\phi}^q_i) (1 - \sigma(u^T U \tilde{\phi}^q_i)) U \tilde{\phi}^q_i (U \tilde{\phi}^q_i)^T
    $$
    Under $\sigma(x)(1 - \sigma(x)) \le 1/4, \forall x$ and boundedness $\Vert \tilde{\phi} \Vert \le B, \Vert U \Vert_F \le R$, 
    $$
        \Vert \nabla^2 g(\cdot) \Vert \le \frac{1}{4} \Vert U \tilde{\phi}^q_i \Vert^2 \le \frac{1}{4}R^2 B^2,
    $$
    which indicates
    \begin{align*}
        \vert L_N(U) - \bar{L}(U) \vert &\le \frac{1}{8}R^2 B^2 \mathbb{E}\Vert \delta \Vert^2  \\
        &= \frac{1}{8N}R^2 B^2 (2B)^2 \\
        &= \frac{R^2 B^4}{2N}.
    \end{align*}

    To show Part (iv), define $h_U(u) = (\sigma(u^T U \tilde{\phi}_i^q) - (1 + z_i^q) / 2) u $. Then, $\nabla L_N(U) = \mathbb{E}[h_U(\hat{\mu}_i) (\tilde{\phi}^q_i)^T]$, $\nabla \bar{L}(U) = \mathbb{E}[h_U(\mu_i)(\tilde{\phi}^q_i)^T]$. After careful calculations and noting $|\sigma'(\dot)| \le 1/4, \sigma''(\cdot) \le 1/6$, we have
    $$
        \Vert \nabla^2_u h_U(u) \Vert_{op} \le c_1 \Vert U \tilde{\phi}_i^q\Vert + c_2\Vert U \tilde{\phi}_i^q\Vert^2 \Vert u \Vert
    $$
    for some positive constants $c_1, c_2 > 0$\footnote{Note that $\nabla_u^2 h_U(\cdot)$ is a tensor lying in $\mathbb{R}^{d \times d \times d}$.}. Since $\mathbb{E}[\nabla h_U(\mu_i) \delta \mid i, \tilde{\phi}_i^q, z_i^q] = \nabla h_U(\mu_i) \mathbb{E}[\delta] = 0$, we have
    \begin{align*}
        \Vert \nabla L_N(U) - \nabla \bar{L}(U) \Vert_F &= \Vert \mathbb{E}[\mathbb{E}[h_U(\hat{\mu}_i) - h_U(\mu_i) \mid i, \tilde{\phi}_i^q, z_i^q] \tilde{\phi}^q_i] \Vert \\
        &\le \frac{1}{2} \mathbb{E}[\Vert \nabla^2 h_U(\mu_i + t\delta) \Vert_{op} \Vert \delta \Vert^2 \Vert \tilde{\phi}^q_i\Vert].
    \end{align*}
    By the almost sure boundedness of $\tilde{\phi}$, $\Vert \mu_i + t \delta \Vert \le \Vert \hat{\mu}_i \Vert + 2 \Vert \mu_i \Vert \le 3B$ almost surely, hence indicating $\Vert \nabla^2 h_U(\mu_i + t\delta) \Vert_{op} \le c_1 RB + 3c_2 R^2 B^3$ almost surely. Therefore,
    \begin{align*}
        \Vert \nabla L_N(U) - \nabla \bar{L}(U) \Vert_F &\le \frac{1}{2}(c_1 RB^2 + 3c_2 R^2 B^4) \mathbb{E}\Vert \delta \Vert^2 \\
        &\le \frac{2 (c_1 + 3c_2 RB^2)B^4}{N}.
    \end{align*}
    According to strong convexity of $\bar{L}(\cdot)$, there is some positive constant $\alpha$ such that
    \begin{align*}
        \Vert U_N^* - \bar{U}^*\Vert_F &\le \frac{1}{\alpha}\Vert \nabla \bar{L}(U_N^*) - \nabla \bar{L}(\bar{U}^*)\Vert_F \\
        &= \Vert \nabla \bar{L}(U_N^*) - \nabla L_N(U_N^*)\Vert_F \\
        &\le \mathcal{O}(1/N)
    \end{align*}
    where the second equality follows $\nabla \bar{L}(\bar{U}^*) = \nabla L_N(U_N^*) = 0$. This concludes the proof.

\end{proof}

\section{Proof of Theorem \ref{thm_failure-ICL}}    \label{apx:proof_thm-fail}
\begin{proof}
    We consider the simplest case where $d=1$ and let $P_{\tilde{\phi}}$ be the following: $P(\tilde{\phi} = +1) = P(\tilde{\phi} = -1) = 1/2$. Then,
    $$
        \mu_{new}(\theta) = \mathbb{E}_{\tilde{\phi}}[\mathrm{tanh}(\frac{1}{2}\tilde{\phi}^T \theta) \tilde{\phi}] = \frac{1}{2} \cdot \mathrm{tanh}(\frac{1}{2} \theta) + \frac{1}{2} \cdot \mathrm{tanh}(-\frac{1}{2} \theta) \cdot (-1) = 
        \mathrm{tanh}(\frac{1}{2} \theta) .
    $$
    Considering the asymptotic case, i.e., $M = \infty, N= \infty$, we have 
    \begin{align*}
        \mathbb{P}_{new}(\hat{z}_{new}^q = 1 \mid \tilde{\phi}) &= \lim_{M \to \infty, N \to \infty} \sigma \left( \frac{1}{M} \sum_{j=1}^M z_{new}^j (\tilde{\phi}^j)^T U_N^* \tilde{\phi}^q  \right) \\
        &= \lim_{M \to \infty}\sigma \left( \frac{1}{M} \sum_{j=1}^M z_{new}^j (\tilde{\phi}^j)^T \bar{U}^* \tilde{\phi}^q  \right) \\
        &= \sigma \left( \mu_{new}(\theta_{new}) ^T \bar{U}^* \tilde{\phi}^q \right)  \\
        &= \sigma \left(\mathrm{tanh}(\frac{1}{2} \theta_{new}) \bar{U}^* \tilde{\phi}^q \right) .
    \end{align*}
    Suppose a correct prediction is achieved asymptotically, i.e., we need to find some $\bar{U}^* \in \mathbb{R}$ such that
    $$
        \mathrm{tanh}(\frac{1}{2} \theta_{new}) \bar{U}^* = \theta_{new}
    $$ 
    which implies that $$\bar{U}^* = \frac{\theta_{new}}{\mathrm{tanh}(0.5 \theta_{new})}.$$
    Moreover, note that $\bar{U}^*$ satisfies $\nabla \bar{L}(\bar{U}^*) = 0$, which is equivalent to
    \begin{align*}
        0 &= \mathbb{E}_{i \sim \mathcal{H}}[(\sigma(\mu_i \bar{U}^*) - \sigma(\theta_i^*))\mu_i] \\
        &= \frac{1}{2}\mathbb{E}_i[(\tanh(0.5 \mu_i \bar{U}^*) - \mathrm{tanh}(0.5 \theta_i^*))\mu_i] \\
        &= \frac{1}{4}\mathbb{E}_i[(\tanh(0.25 \tanh(0.5 \theta_i^*) \bar{U}^*) - \mathrm{tanh}(0.5 \theta_i^*))\tanh(0.5 \theta_i^*)]
    \end{align*}
    where we use $\mu_i = \frac{1}{2}\tanh(\frac{1}{2}\theta_i^*)$ and $\sigma(x) = \frac{1}{2}\mathrm{tanh}(\frac{1}{2} x) + \frac{1}{2}$. Let $m = \tanh(0.5 \theta_i^*) \in (-1, 1)$. Due to monotonicity and invertibility of $\tanh(\cdot)$, there exists some distribution $\theta_i^* \sim \mathcal{H}$ whose support is $\mathbb{R}$ such that $m \sim \text{Unif}[-1, 1]$ where $\text{Unif}[\cdot, \cdot]$ denotes the uniform distribution. Thus, it is equivalent that $\bar{U}^*$ satisfies 
    \begin{align*}
        0 &= \mathbb{E}_{m}[m(\tanh(0.25m \bar{U}^*) - m)] \\
        &= \mathbb{E}_{m}[m(\tanh(0.25m \bar{U}^*)] - \frac{1}{3} \\
        &= \frac{1}{2}\int_{-1}^1 m(\tanh(0.25m \bar{U}^*) dm - \frac{1}{3}
    \end{align*}
    Define
    $$
        I(u) = \frac{1}{2}\int_{-1}^1 m(\tanh(0.25m \bar{U}^*) dm.
    $$
    It is obvious that $I'(u) > 0, \forall u $ which indicates $I(u)$ is monotonically increasing. Noticing that $I(0) < 1/3$ $I(6) > 1/3$, it implies $\bar{U}^* \in (0,6)$. Further, as $\bar{U}^* = \frac{\theta_{new}}{\mathrm{tanh}(0.5 \theta_{new})}$ in order for asymptotically correct prediction, considering $\forall \theta_{new} > 6 $, we have $\bar{U}^* > 6$, which contradicts to the previous claim $\bar{U}^* \in (0,6)$. Therefore, we conclude the proof as we find distributions $P_{\tilde{\phi}}, \mathcal{H}$ and $\theta_{new}$ such that $\mathbb{P}(\hat{z}^q_{new} = \cdot \mid \tilde{\phi}^q) \ne \mathbb{P}(z^q_{new} = \cdot \mid \tilde{\phi}^q)$.
\end{proof}

\section{Proof of Theorem \ref{thm_res-time-conv}}  \label{apx:proof_thm-res}
\begin{proof}
    Define $\hat{s}_i = \frac{1}{N}\sum_{l=1}^N (z_i^l / t_i^l) \tilde{\phi}_i^l$, $s_i = 2\Sigma_{\tilde{\phi}}\theta_i^*$ and $\mu_{z_i^l} = \mathbb{E}[z_i^l \mid \tilde{\phi}_i^l]$, $\mu_{t_i^l} = \mathbb{E}[t_i^l \mid \tilde{\phi}_i^l]$. Then,
    \begin{align*}
        \hat{s}_i - s_i &= \frac{1}{N}\sum_{l=1}^N \left(\frac{z_i^l}{t_i^l} - \frac{\mu_{z_i^l}}{\mu_{t_i^l}} \right) \tilde{\phi}_i^l + \frac{1}{N}\sum_{l=1}^{N} \left(\frac{\mu_{z_i^l}}{\mu_{t_i^l}} \tilde{\phi}_i^l - 2\Sigma_{\tilde{\phi}}\theta_i^* \right) \\
        &=\frac{1}{N}\sum_{l=1}^N \left(\frac{z_i^l}{t_i^l} - 2(\tilde{\phi}_i^l)^T \theta_i^* \right) \tilde{\phi}_i^l + \frac{1}{N}\sum_{l=1}^{N} \left(2(\tilde{\phi}_i^l)^T \theta_i^* \tilde{\phi}_i^l - 2\Sigma_{\tilde{\phi}}\theta_i^* \right).
    \end{align*}
    Taking expectation of its norm square yields
    \begin{align*}
        \mathbb{E}\Vert \hat{s}_i - s_i \Vert^2 &\le \frac{2}{N}\sum_{l=1}^N \mathbb{E}\left[\left(\frac{z_i^l}{t_i^l} - 2(\tilde{\phi}_i^l)^T \theta_i^* \right)^2 \Vert \tilde{\phi}_i^l \Vert^2 \right] + \frac{2}{N}\mathbb{E}\left[\mathrm{Var}(\Vert 2(\tilde{\phi}_i^l)^T \theta_i^* \tilde{\phi}_i^l \Vert \mid \theta_i^*) \right] \\
        &\le \mathcal{O}(1/K) + \mathcal{O}(1/N).
    \end{align*}
    where we use Lemma \ref{lmm_frac-conv} in the last inequality.
    It is straightforward that
    \begin{align}
        \text{vec}(\nabla L_{N,K}(U)) &= \mathbb{E}\left[ (\hat{s}_i^T U \tilde{\phi}_i^q - z_i^q / t_i^q) \tilde{\phi}_i^q \otimes \hat{s}_i \right] \nonumber \\
        \text{vec}(\nabla L_{\infty}(U)) &= \mathbb{E}\left[ (s_i^T U \tilde{\phi}_i^q - 2(\tilde{\phi}_i^q)^T \theta_i^*) \tilde{\phi}_i^q \otimes s_i \right] \nonumber
    \end{align}
    and hence
    \begin{align*}
        \nabla L_{N,K}(U) - \nabla L_{\infty}(U) &= \underbrace{\mathbb{E}\left[ ((\hat{s}_i - s_i)^T U \tilde{\phi}_i^q - (z_i^q / t_i^q - 2 (\tilde{\phi}_i^q)^T \theta_i^*)) (\hat{s}_i - s_i)(\tilde{\phi}_i^q)^T  \right]}_{e_1} \\
        &~ + \underbrace{\mathbb{E}\left[ (s_i^T U \tilde{\phi}_i^q - 2 (\tilde{\phi}_i^q)^T \theta_i^*)(\hat{s}_i - s_i)(\tilde{\phi}_i^q)^T \right]}_{e_2} \\
        &~ + \underbrace{\mathbb{E}\left[ (((\hat{s}_i - s_i)^T U \tilde{\phi}_i^q) - (z_i^q / t_i^q - 2 (\tilde{\phi}_i^q)^T \theta_i^*)) s_i (\tilde{\phi}_i^q)^T \right]}_{e_3}.
    \end{align*}
    We first bound $\Vert e_1 \Vert_F$:
    \begin{align*}
        \Vert e_1 \Vert_F &\le R B^2 \mathbb{E}\left\Vert \hat{s}_i - s_i \right\Vert^2 + B\mathbb{E}\left[\Vert z_i^q / t_i^q - 2(\tilde{\phi}_i^q)^T \theta_i^* \Vert \Vert \hat{s}_i - s_i \Vert \right] \\
        &\le R B^2 \mathbb{E}\left\Vert \hat{s}_i - s_i \right\Vert^2 + B \sqrt{\mathbb{E}\Vert z_i^q / t_i^q - 2(\tilde{\phi}_i^q)^T \theta_i^* \Vert^2} \cdot \sqrt{\mathbb{E}\Vert \hat{s}_i - s_i \Vert^2} \\
        &\le \mathcal{O}(1/N + 1/K + 1/\sqrt{NK})
    \end{align*}
    where we use Lemma \ref{lmm_frac-conv} in the last inequality. For $e_2$ we have
    \begin{align*}
        \Vert e_2 \Vert_F &\le B\sqrt{\mathbb{E}\Vert (\theta_i^*)^T (\Sigma_{\tilde{\phi}}U - I) \tilde{\phi}_i^q \Vert^2} \cdot \sqrt{\mathbb{E}\Vert \hat{s}_i - s_i \Vert^2}  \\
        &= B\sqrt{\mathbb{E} [(\tilde{\phi}_i^q)^T(\Sigma_{\tilde{\phi}}U - I)^T \theta_i^* (\theta_i^*)^T (\Sigma_{\tilde{\phi}}U - I) \tilde{\phi}_i^q]} \cdot \sqrt{\mathbb{E}\Vert \hat{s}_i - s_i \Vert^2} \\
        &\le B^2 (R\text{tr}(\Sigma_{\tilde{\phi}}) + d)\text{tr}(\mathbb{E}[\theta_i^* (\theta_i^*)^T]) \cdot \mathcal{O}(1/\sqrt{N} + 1/\sqrt{K}) \\
        &= \mathcal{O}(1/\sqrt{N} + 1/\sqrt{K}).
    \end{align*}
    Similarly, we can show $\Vert e_3 \Vert_{F} \le \mathcal{O}(1/\sqrt{N} + 1/ \sqrt{K})$. Thus, $\Vert  \nabla L_{N,K}(U) - \nabla L_{\infty}(U)\Vert_F \le \mathcal{O}(1/\sqrt{N} + 1/\sqrt{K})$.

    Next, we show that $L_{\infty}$ is strongly-convex. Note that
    \begin{align*}
        \nabla^2_{\text{vec}(U)} L_{\infty}(U) &= \mathbb{E}[\tilde{\phi}_i^q (\tilde{\phi}_i^q)^T \otimes s_i s_i^T] \\
        &= 4\mathbb{E}[\tilde{\phi}_i^q (\tilde{\phi}_i^q)^T \otimes  \Sigma_{\tilde{\phi}} \theta_i^* (\theta_i^*)^T \Sigma_{\tilde{\phi}}] \\
        &= 4\mathbb{E}[\tilde{\phi}_i^q (\tilde{\phi}_i^q)^T] \otimes [\Sigma_{\tilde{\phi}} \mathbb{E}(\theta_i^* (\theta_i^*)^T) \Sigma_{\tilde{\phi}}].
    \end{align*}
    Since both $\mathbb{E}[\tilde{\phi}_i^q (\tilde{\phi}_i^q)^T]$ and $\mathbb{E}[\theta_i^* (\theta_i^*)^T]$ are full-rank, $\nabla^2_{\text{vec}(U)} L_{\infty}(U) \succ 0$, implying $L_{\infty}$ is $\bar{\alpha}$-strongly-convex for some $\alpha > 0$. Similarly, one can show $L_{N,K}$ is strongly-convex. As $U^* = \Sigma_{\tilde{\phi}}^{-1}$ is a minimizer of $L_{\infty}$ (which is obviously obtained by $\nabla L_{\infty}(\Sigma_{\tilde{\phi}}^{-1}) = 0$), strong convexity yields $U^*$ is unique.

    By strong convexity of $L_{\infty}$, 
    \begin{align*}
        \Vert U_{N,K}^* - U^* \Vert_F &\le \frac{1}{\bar{\alpha}}\Vert \nabla L_{\infty}(U_{N,K}) - \nabla L_{\infty}(U^*) \Vert_F  \\
        &= \frac{1}{\bar{\alpha}}\Vert \nabla L_{\infty}(U_{N,K}^*) - \nabla L_{N,K}(U_{N,K}^*) \Vert_F \\
        &\le \mathcal{O}(1/\sqrt{N} + 1/\sqrt{K}),
    \end{align*}
    where the second equality follows $U_{N,K}^*$ and $U^*$ are minimizers of $L_{N,K}$ and $L_{\infty}$ by definition. This completes the proof.
    
\end{proof}

\section{Proof of Corollary \ref{coro_correct-pred}}    \label{apx:proof_coro}
\begin{proof}
    Let $\hat{s}_{new} = \frac{1}{M} \sum_{l=1}^M \frac{z_{new}^l}{t_{new}^l} \tilde{\phi}^l, s_{new} = 2 \Sigma_{\tilde{\phi}}\theta_{new}$. Similar to the proof of Theorem \ref{thm_res-time-conv}, we can show
    $$
        \mathbb{E}\Vert \hat{s}_{new} - s_{new} \Vert^2 \le \mathcal{O}(1/M + 1/K).
    $$
    Then,
    \begin{align*}
        \mathbb{E}(\hat{o}_{new}^q - 2\theta_{new}^T \tilde{\phi}^q)^2 &= \mathbb{E}(\hat{s}_{new}^T U_{N,K}^* \tilde{\phi}^q - s_{new}^T U^* \tilde{\phi}^q)^2 \\
        &\le 3\mathbb{E}[(\hat{s}_{new} - s_{new})^T(U_{N,K} - U^*) \tilde{\phi}^q]^2 + 3\mathbb{E}[s^T(U_{N,K} - U^*) \tilde{\phi}^q]^2 \\
        &~ + 3\mathbb{E}[(\hat{s}_{new} - s_{new})^T U^* \tilde{\phi}^q]^2 \\
        &\le \mathcal{O}(1/M + 1/N + 1/K + 1/\sqrt{NK})
    \end{align*}
    which completes the proof.
\end{proof}

\section{A Useful Lemma}
\begin{lemma}   \label{lmm_frac-conv}
    For i.i.d. $\{ X_i \}_{i=1}^N$ and i.i.d. $\{ Y_i \}_{i=1}^N$, define $\bar{X} = \frac{1}{N}\sum_{i=1}^N X_i$, $\bar{Y} = \frac{1}{N}\sum_{i=1}^N Y_i$ and $\mu_X = \mathbb{E}[X_1]$, $\mu_Y = \mathbb{E}[Y_1]$. Then 
    $$
        \mathbb{E}\left(\frac{\bar{X}}{\bar{Y}} - \frac{\mu_X}{\mu_Y} \right)^2 = \mathcal{O}(N^{-1}).
    $$
\end{lemma}
\begin{proof}
    Let $\delta_X = \bar{X} - \mu_X$, $\delta_Y = \bar{Y} - \mu_Y$.
    \begin{align*}
        \frac{\bar{X}}{\bar{Y}} &= \frac{\mu_X + \delta_X}{\mu_Y}\left( 1 + \frac{\delta_Y}{\mu_Y} \right)^{-1} \\
        &= \frac{\mu_X + \delta_X}{\mu_Y}\left( 1 - \frac{\delta_Y}{\mu_Y} + \frac{\delta_Y^2}{\mu_Y^2} + \text{higher-order terms} \right) \\
        &= \frac{\mu_X}{\mu_Y} + \frac{\delta_X}{\mu_Y} - \frac{\mu_X \delta_Y}{\mu_Y^2} + \text{higher-order terms}.
    \end{align*}
    Let $\Delta = \frac{\bar{X}}{\bar{Y}} - \frac{\mu_X}{\mu_Y}$. We have
    \begin{align*}
        \Delta^2 &= \left( \frac{\delta_X}{\mu_Y} - \frac{\mu_X \delta_Y}{\mu_Y^2} \right)^2 \\
        &= \frac{\delta_X^2}{\mu_Y^2} + \frac{\mu_X^2 \delta_Y^2}{\mu_Y^4} - \frac{\mu_X \delta_X \delta_Y}{\mu_Y^3} + \text{higher-order terms}.
    \end{align*}
    Since $\mathbb{E}[\delta_X^2] = \mathrm{Var}(X) =  \mathcal{O}(1/N), \mathbb{E}[\delta_Y^2] = \mathrm{Var}(Y) =  \mathcal{O}(1/N), \mathbb{E}[\delta_X \delta_Y] = \mathrm{Cov}(X, Y) =  \mathcal{O}(1/N)$, we conclude that
    $$
        \mathbb{E}[\Delta^2] = \mathcal{O}(1/N).
    $$
\end{proof}

\end{document}